%% file: paper.tex
\pdfoutput=1

\documentclass[11pt]{article}

\usepackage{acl}

\usepackage{times}
\usepackage{latexsym}
\usepackage[most]{tcolorbox}
\usepackage{enumitem}

\usepackage[T1]{fontenc}

\usepackage[utf8]{inputenc}

\usepackage{microtype}

\usepackage{inconsolata}

\usepackage{graphicx}

\usepackage{hyperref}

\usepackage{amsmath}
\usepackage{amssymb}
\usepackage{booktabs}

\usepackage{listings}
\usepackage{xcolor}       
\usepackage{multirow}

\colorlet{lightyellow}{yellow!40}
\definecolor{navyblue}{rgb}{0.0, 0.0, 0.5}

\colorlet{lightyellow}{yellow!40}

\makeatletter
{\small 
\xdef\f@size@small{\f@size}
\xdef\f@baselineskip@small{\f@baselineskip}
\normalsize 
\xdef\f@size@normalsize{\f@size}
\xdef\f@baselineskip@normalsize{\f@baselineskip}
}
\newcommand{\smalltonormalsize}{%
  \fontsize
    {\fpeval{(\f@size@small+\f@size@normalsize)/2}}
    {\fpeval{(\f@baselineskip@small+\f@baselineskip@normalsize)/2}}%
  \selectfont
}
\makeatother

\usepackage{amsmath, amssymb, amsthm}
\newtheorem{theorem}{Theorem}
\newtheorem{definition}{Definition}
\newtheorem{corollary}{Corollary}
\newtheorem{hypothesis}{Hypothesis}

\title{Student Data Paradox and Curious Case of Single Student-Tutor Model: \\
Regressive Side Effects of Training LLMs for Personalized Learning}

\author{
  Shashank Sonkar, Naiming Liu, Richard G. Baraniuk \\
  Rice University \\
  \texttt{shashank.sonkar@rice.edu} \\
}

\begin{document}
\maketitle
\begin{abstract}
The pursuit of personalized education has led to the integration of Large Language Models (LLMs) in developing intelligent tutoring systems. 
To better understand and adapt to individual student needs, including their misconceptions, LLMs need to be trained on extensive datasets of  student-tutor dialogues.
Our research uncovers a fundamental challenge in this approach: the ``Student Data Paradox.''
This paradox emerges when LLMs, trained on student data to understand learner behavior, inadvertently compromise their own factual knowledge and reasoning abilities.
We investigate this paradox by training state-of-the-art language models on student-tutor dialogue datasets and evaluating their performance across multiple benchmarks. 
These benchmarks assess various aspects of language model capabilities, including reasoning, truthfulness, and common sense understanding. 
Our findings reveal significant declines in the models' performance across these diverse benchmarks, indicating a broad impact on their capabilities when trained to model student behavior.
Our research makes two primary contributions: (1) empirical demonstration of the Student Data Paradox through quantitative analysis of model performance, and (2) introduction of ``hallucination tokens'' as a mitigation strategy. 
These tokens, while improving performance, highlight the persistent challenge of balancing accurate student behavior modeling with maintaining the LLM's integrity as an educational tool.
This study emphasizes the need for innovative solutions to reconcile the conflicting goals of faithfully understanding diverse student cognition while preserving the model's ability to provide accurate information and guidance.

\end{abstract}

\section{Introduction}
\input{sections/introduction}

\section{Theorem: The Case for a Single Student-Tutor Model}
\input{sections/theorem}

\section{Methodology}
\input{sections/methodology}

\section{Experiments and Discussion}
\input{sections/experiments}

\section{Related Work}
\input{sections/related_work}

\section{Conclusion}
\input{sections/conclusion}

\section{Limitation}
\input{sections/limitations}

\section{Ethics and Risk}
\input{sections/ethics}

\section*{Acknowledgments}
This work was supported by NSF grant 1842378, ONR grant N0014-20-1-2534, AFOSR grant FA9550-22-1-0060, a Vannevar Bush Faculty Fellowship, and ONR grant N00014-18-1-2047.

\bibliography{custom}

\end{document}

%% file: sections/introduction.tex
Personalized education is undergoing a significant transformation, driven by the rapid advancement of Large Language Models (LLMs) \cite{learnlm}.
These AI systems are being developed to serve as adaptive tutors, capable of understanding and responding to individual student needs \cite{openai2024chatgptedu,nlet,sonkar2024pedagogical}.
The incorporation of student data into the training process is an important step towards creating truly personalized learning experiences. 
By training LLMs on diverse student data, including dialogues, errors, and problem-solving approaches, these models can learn to identify knowledge gaps, recognize common misconceptions, and provide personalized support tailored to each learner's unique needs.
However, this approach is at odds with traditional LLM training, where data quality is of utmost importance \cite{goldbergAssessingClaims2022,samaniegoSanityCheckingLarge2022}.
Student data, by nature, is prone to errors and misconceptions, which raises concerns about the impact on the LLMs' factual knowledge and reasoning abilities.
This tension between the need to understand student misconceptions accurately and the imperative to maintain the model's factual integrity creates a paradoxical challenge, which we term the ``\textbf{Student Data Paradox.}''

In this paper, we study this timely topic: what happens when we train an LLM on student-tutor interaction data? Our research reveals that while training LLMs on student data can enable models to accurately simulate authentic student behaviors and misconception, it comes at a significant cost -- LLMs' own factual integrity and reasoning capabilities are compromised. This paradox poses a serious concern, as the primary purpose of any educational model is to provide accurate and reliable information to learners.

\input{tables/figure_hal}

We investigated this paradox by fine-tuning Llama \cite{llama2} and Vicuna models \cite{vicuna2023}, on the CLASS dataset \cite{sonkar2023class} - a collection of student-tutor dialogues on college-level biology questions (example in figure~\ref{fig:tutorbot}).
Training LLMs to simulate student behavior significantly decreased the models' performance across various benchmark datasets, including the ARC reasoning challenge, TruthfulQA, Hallucination Evaluation Dialogue, and MemoTrap \cite{clark2018thinkarc,lin2022truthfulqa,li-etal-2023-halueval,mckenzie2022inverse}. 
We present a detailed analysis across nine key benchmarks using the Eleuther LLM leaderboard \cite{eval-harness} to provide a thorough assessment of LLMs' reasoning and general knowledge capabilities.
In contrast, a control experiment revealed that training LLMs to simulate tutor behavior did not lead to similar performance declines. This finding highlights that the observed regressive effects are uniquely associated with training LLMs to replicate student misconceptions.

A seemingly intuitive solution to the paradox would be to create separate student and tutor models. 
However, we present a formal theorem demonstrating that this approach necessitates potentially vast number of student models, making it impractical.
This finding emphasizes the need for a unified framework to capture both accurate knowledge and diverse student misconceptions.
To this end to counteract the side effects, we propose to incorporate novel start and end hallucination tokens (\texttt{[hal]} and \texttt{[/hal]}) into the LLM training process. 
These tokens, placed at the beginning and end of each student response, serve as cues to the model, instructing it when to differentiate between providing accurate responses and replicating student misconceptions. 
Our results indicate a substantial improvement in the model's performance across all datasets after introducing this token.
However, these tokens do not fully restore the model's baseline performance, underscoring the complexity of the issue. 

\textbf{Contributions:} Our research makes important contributions to the use of LLMs in modeling student behavior for adaptive tutoring systems. We \textbf{identify and study the Student Data Paradox} in LLMs trained on student-tutor interactions, revealing a trade-off between accurately modeling student misconceptions and maintaining the LLM's factual integrity and reasoning ability (section~\ref{sec:method} and ~\ref{sec:exp}). Through a formal theorem discussed in section~\ref{sec:theorem}, we show that the intuitive \textbf{solution of separating student and tutor models is fundamentally flawed}, necessitating a more sophisticated, unified approach. We propose hallucination tokens as a method for this unified approach. These tokens, integrated into the training process, enable LLMs to differentiate between modeling student misconceptions and providing accurate information, significantly improving model performance.

Our study exposes the critical trade-offs in training LLMs on student data and offers a concrete solution with hallucination tokens. This work lays the foundation for developing more robust and effective AI tutors that can accurately model student behavior without compromising factual integrity.

%% file: tables/figure_hal.tex
\begin{figure*}[t!]
    \centering
    \includegraphics[width=1.25\columnwidth]{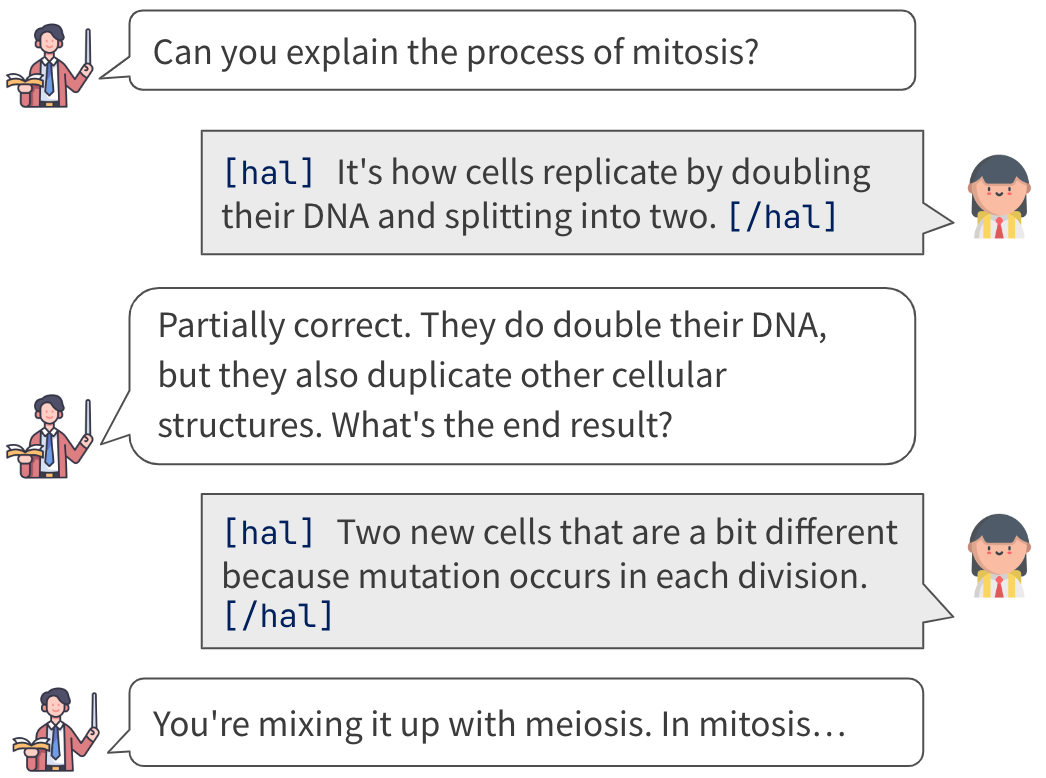}
    \caption{
    Training LLMs on student-tutor dialogues to model student behavior: an example from the CLASS dataset. This approach creates ``Student Data Paradox'': while the model accurately simulates student responses and misconceptions, its own factual knowledge and reasoning abilities are compromised. We introduce hallucination tokens (\texttt{[hal]}, \texttt{[/hal]}) during training to instruct the model when to replicate misconceptions versus provide accurate information, showing promise in mitigating the paradox.
    }
    \label{fig:tutorbot}
\end{figure*}

%% file: sections/theorem.tex
\label{sec:theorem}

To address the ``Student Data Paradox'', one may propose a seemingly straightforward solution - the creation of two distinct models: a student model and a tutor model. This approach seems intuitive: one model to capture diverse student behaviors and misconceptions, and another to represent the ideal tutor with perfect knowledge.
However, this seemingly straightforward solution harbors significant complexities and challenges that are not immediately apparent. Our goal is to show that what may appear to be a simple dichotomy between student and tutor models necessarily evolves into a multi-model approach with potentially multitude of student models.

To illustrate the complexity of this problem, consider a simple algebraic equation: $Ax = B$. A tutor (or a student with perfect knowledge) would correctly solve this as $x = B/A$. However, different students might have various misconceptions, leading to incorrect solutions such as $x = A/B$, $x = B - A$, or $x = B + A$. The proposal to separate tutor and student models aimed to resolve the inconsistency between correct and incorrect knowledge. However, this solution overlooks an important issue: the inconsistencies among the incorrect rules themselves. These mutually contradictory misconceptions cannot be accurately represented by a single student model without compromising the model's internal consistency. Consequently, to faithfully represent the diverse range of student misconceptions, we would need multiple student models, each consistently representing a different set of misconceptions. This realization demonstrates that the apparent dichotomy between a tutor model and a student model inevitably evolves into a multi-model scenario, with potentially numerous student models. Thus, the initial attempt to resolve inconsistencies by separating correct and incorrect knowledge actually leads to a proliferation of models, highlighting the need for a more sophisticated, unified approach to effectively capture both accurate knowledge and diverse student misconceptions.

The complexity illustrated in the previous example underscores the need for a more rigorous, formal treatment of the problem. To proceed with our analysis, we must first establish clear definitions and state our working hypothesis. These will form the foundation for our main theorem and subsequent proofs, allowing us to formally demonstrate why the apparent dichotomy between student and tutor models inevitably leads to a multi-model scenario.

\subsection{Definitions}

\begin{definition}
Let $R$ be the set of all possible rules or misconceptions a student might have.
\end{definition}

\begin{definition}
Two rules $r_1, r_2 \in R$ are considered logically consistent if and only if there exists a possible state of the world in which both $r_1$ and $r_2$ can be simultaneously true or applicable.
\end{definition}

\begin{definition}
For any $R_i, R_j \subseteq R$, let $C(R_i, R_j)$ be a function that returns true if all rules in $R_i$ are logically consistent with all rules in $R_j$, and false otherwise.
\end{definition}

\begin{definition}
Let $S$ be the set of all distinct student models, where each model $M_i \in S$ represents a unique, internally consistent set of rules or misconceptions. Formally,
$S = \{M_i : M_i \text{ represents a maximal set } R_i \subseteq R \text{ such that } \forall r_1, r_2 \in R_i, C(\{r_1\}, \{r_2\}) = true\}$
where $R$ is the set of all possible rules or misconceptions.
\end{definition}

\begin{definition}
A "perfect student" is defined as one who consistently provides correct answers and possesses comprehensive subject understanding, functionally equivalent to a tutor in terms of knowledge representation.
\end{definition}

\begin{hypothesis}
Models struggle to simultaneously represent multiple logically inconsistent rules while maintaining internal consistency.
\end{hypothesis}

If this hypothesis is false, it implies that a single student-teacher model is theoretically feasible and supports our unified approach of a single student-tutor model.


\subsection{Main Theorem}

\begin{theorem}
The apparent dichotomy between a student model and a tutor model necessarily evolves into a multi-model approach with potentially numerous student models.
\end{theorem}

\begin{proof}
1. Assume, for contradiction, that a single student model $M_s$ can accurately represent all student misconceptions.

2. Consider two sets of misconceptions $R_1, R_2 \subset R$ such that $C(R_1, R_2) = false$.

3. By our hypothesis, $M_s$ cannot consistently represent both $R_1$ and $R_2$ simultaneously.

4. Therefore, we need at least two distinct student models, contradicting our assumption of a single student model.

5. Let $n$ be the number of maximal subsets of $R$ containing mutually consistent rules.

6. We require $n$ distinct models to represent all sets of consistent misconceptions.

7. Let $R_c \subset R$ be the set of correct rules. One of these models, $M_c$, corresponds to $R_c$, representing both a perfect student and a tutor.

8. Therefore, the total number of necessary models is $n$, where $n \geq 2$ and one model serves as both the perfect student and tutor model.
\end{proof}

\subsection{Corollaries}

\begin{corollary}
The number of required student models is at least equal to the number of maximal sets of mutually consistent misconceptions.
\end{corollary}

\begin{corollary}
One model in the set of student models is functionally equivalent to the tutor model.
\end{corollary}

\begin{corollary}
The apparent two-model approach (student and tutor) actually requires $n$ models, where $n$ can be large, with one of these models serving as both the perfect student and tutor model.
\end{corollary}

This formal treatment demonstrates that the intuitive separation of student and tutor models is insufficient to address the complexities of representing diverse student misconceptions. The proliferation of necessary models to accurately capture different sets of consistent misconceptions reveals the impracticality of this approach.
These findings point to the necessity of a unified approach to model training that can efficiently represent both accurate knowledge and diverse student misconceptions within a single framework. Such an approach would need to address the challenge of maintaining internal consistency while capturing a wide range of potentially conflicting rules.

In the following section, we propose a novel solution to address this challenge: the incorporation of hallucination tokens into the training process of LLMs. This approach aims to provide a unified framework that can represent both accurate knowledge and diverse student misconceptions without resorting to multiple, separate models.

%% file: sections/methodology.tex
\label{sec:method}
Our methodology is divided into three main parts: data preparation, model training, and the incorporation of hallucination tokens. 

\subsection{Data Preparation}

The first step in our methodology involves preparing the dataset for training the LLMs. We denote the conversation dataset as \( \mathcal{D} \), which consists of ordered pairs of tutor-student conversational turns: \(\mathcal{D} = \{(\mathbf{x}_1, \mathbf{y}_1), (\mathbf{x}_2, \mathbf{y}_2), \ldots, (\mathbf{x}_N, \mathbf{y}_N)\}\), where \(N\) is the total number of conversational turns. Each \( \mathbf{x} \) represents a sequence of tutor utterances, and each corresponding \( \mathbf{y} \) represents the student response. 

The dataset is derived from the CLASS framework \cite{sonkar2023class}, which provides a realistic representation of student learning patterns, featuring student misconceptions and the tutor's rectifications. This dataset provides a rich source of student-tutor dialogues on biology questions sourced from college textbooks. 

\subsection{Model Training}

The second step in our methodology involves training LLMs. The LLMs are designed to predict the next utterance given the previous conversational context. Unlike traditional approaches that focus on the correct responses typically output by a tutoring system, our model centers on student outputs, which may possess a mix of correctness and misconceptions. 

For an input sequence \( \mathbf{x}_i \), the LLM aims to generate an output sequence \( \hat{\mathbf{y}}_i \) that resembles a student's response. The language modeling loss for a single data pair is defined by the negative log likelihood:

\[
\mathcal{L} (\mathbf{y}_i, \hat{\mathbf{y}}_i) = - \sum_{t=1}^{|\mathbf{y}_i|} \log p \left( y_{i,t} \middle| \mathbf{x}_i, \mathbf{y}_{i,<t}; \theta \right)
\]

where \( \mathbf{y}_{i,<t} \) indicates the tokens in the true response preceding the current token \( y_{i,t} \), and \( \theta \) encapsulates the parameters of the LLM. The overall training loss is the sum over the entire dataset:

\[
\mathcal{L}_{\text{total}} = \sum_{i=1}^{N} \mathcal{L} \left( \mathbf{y}_i, \hat{\mathbf{y}}_i \right)
\]

\subsection{Incorporation of Hallucination Tokens}

The third step in our methodology involves the incorporation of hallucination tokens. To enhance the LLM's ability to generate responses that simulate student behaviors, including providing incorrect or uncertain information, we introduce hallucination token markers. Each student response in the dataset is enriched with these markers to indicate the beginning and the end of the potentially inaccurate content. 

Let \( \mathbf{y}_i \) be an original student response sequence from the dataset. The augmented student response \( \tilde{\mathbf{y}}_i \) used for training is constructed by prepending and appending hallucination tokens \texttt{[hal]} and \texttt{[/hal]}, respectively:

\[
\tilde{\mathbf{y}}_i = \left[ \texttt{[hal]}, \mathbf{y}_{i,1}, \mathbf{y}_{i,2}, \ldots, \mathbf{y}_{i,|\mathbf{y}_i|}, \texttt{[/hal]} \right]
\]

In the modified training regime, the LLM predicts the sequence \( \hat{\mathbf{y}}_i \) such that it learns to include these tokens, effectively grasping the context of student uncertainty or errors.
These tokens serve as cues to the model, instructing it when to differentiate between providing accurate responses and replicating student misconceptions. 

%% file: sections/experiments.tex
\label{sec:exp}
\input{tables/truthfulqa}

In this section, we present our experimental methodology and discuss the findings in detail.
The experiments were designed to explore the regressive side effects of training LLMs to model student behavior and to assess the effectiveness of our proposed hallucination tokens in mitigating these effects. 

\subsection{Experimental Setup}

We trained the Vicuna 7B and 13B models \cite{vicuna2023}, one of the best open-source LLMs, on a student-tutor dialogue dataset derived from the CLASS \cite{sonkar2023class} framework.
This dataset, which provides a realistic representation of student learning patterns, misconceptions, and the tutor's rectifications, was used to fine-tune the models to generate outputs that model student dialogue. The dataset contains $648$ conversations, which sums up to a total of 20K student-tutor interactions. Average length of conversations is around $~400$ words, only including the student and tutor fields in the conversation template.

The models were evaluated across seven key benchmarks using the Eleuther AI Language Model Evaluation Harness \cite{eval-harness}. 
These benchmarks include the TruthfulQA \cite{lin2022truthfulqa}, ARC \cite{clark2018thinkarc}, HellaSwag \cite{zellers2019hellaswag}, Winogrande \cite{WINOGRANDE}, MMLU \cite{mmlu}, HaluEval Dialogue \cite{li-etal-2023-halueval}, and MemoTrap \cite{mckenzie2022inverse}. 
Each of these benchmarks tests different aspects of the model's performance, including its truthfulness, reasoning abilities, ability to recognize hallucinations, and memory-based task performance.

\subsection{In-depth Analysis: TruthfulQA}

In the realm of educational technology, the veracity of information provided by a model is of paramount importance. Misinformation or misconceptions can lead to significant learning detriments, making the truthfulness of a model's responses a critical factor in its effectiveness as an educational tool. Therefore, we chose to conduct an in-depth analysis of our models' performance on the TruthfulQA benchmark. 

TruthfulQA is a benchmark specifically designed to measure the truthfulness of a language model's responses across a wide range of categories. It tests the model's ability to avoid generating false answers learned from imitating human texts, a challenge that is particularly relevant to our study. Given the importance of truthfulness in educational contexts and the unique challenges posed by training models to model student misconceptions, we believe that a rigorous analysis of our models' performance on TruthfulQA is warranted.

In this section, we present our findings from the TruthfulQA benchmark, exploring the impact of training models to model student behavior and the effectiveness of our proposed hallucination tokens in mitigating any negative effects. We delve into the results from the multiple-choice and generation tasks within TruthfulQA, providing a comprehensive view of our models' truthfulness in different contexts.

\textbf{TruthfulQA Multiple-Choice Setting 1 (MC1) Findings.}
In the first multiple-choice setting, where there is a single correct label, the student-7b model's accuracy decreased by 15 points compared to the vicuna-7b model. However, the introduction of hallucination tokens led to a significant recovery in performance. This finding is particularly relevant in the context of education, where maintaining the truthfulness of responses is crucial. The improvement with hallucination tokens suggests that it is possible to train models that can both simulate student behavior and adhere to factual accuracy, a key consideration for deploying LLMs in educational settings.

\textbf{TruthfulQA Multiple-Choice Setting 2 (MC2) Findings.}
In the second multiple-choice setting, where multiple correct labels are possible, we observed a similar trend to the MC1 setting. The student-7b model experienced a significant drop in accuracy, from 50.37\% in the vicuna-7b model to 36.14\% when trained to model student responses. However, the introduction of hallucination tokens led to a notable improvement in performance, with the student-7b model's accuracy recovering to 44.68\%. 

This recovery is particularly relevant in the context of education, where multiple perspectives or answers might be correct. The ability of the model to navigate such complexities while maintaining truthfulness is crucial. The improvement with hallucination tokens suggests that it is possible to train models that can both simulate student behavior and adhere to factual accuracy, a key consideration for deploying LLMs in educational settings.

\textbf{TruthfulQA Generation Findings.}
For the TruthfulQA generation task, where the model is tasked with generating 1-2 sentence answers, we employed ROUGE scores to evaluate performance due to the generative nature of the task. The student-7b model saw a significant decrease in ROUGE scores, from 51.41 in the vicuna-7b model to 29.74, indicating a substantial loss in the ability to generate truthful, relevant responses. However, the introduction of hallucination tokens led to a significant recovery in performance, with ROUGE scores improving to 47.61.

This finding is crucial for educational technology as LLMs are increasingly used as generative agents to create educational content, provide explanations, and engage in dialogue with students. The ability to generate truthful, accurate responses is fundamental to their utility in these contexts. The recovery observed with hallucination tokens highlights their potential to enable LLMs to simulate student misconceptions for personalized learning without sacrificing the quality and truthfulness of their output.

\input{tables/results_halu_mt}

\subsection{Benchmark Evaluation}

Following the exploration of TruthfulQA settings, we delve into the performance of our models across a broader range of benchmarks as detailed in Table~\ref{tab:results}. These benchmarks—ARC, HaluEval Dial, MemoTrap, MMLU, HellaSwag, and Winogrande—offer a comprehensive view of the models' capabilities in reasoning, detecting hallucinations, avoiding memorization traps, and understanding commonsense, respectively.

\textbf{AI2 Reasoning Challenge (ARC) Findings.}
ARC serves as a rigorous benchmark to evaluate a model's reasoning capabilities through a set of grade-school science questions. These questions are designed to test not just the factual knowledge of the models but also their ability to apply this knowledge in reasoning through complex, multi-step problems. The ARC dataset is particularly relevant in educational contexts as it mirrors the type of critical thinking and problem-solving skills students are expected to develop.

In our experiments, the performance of models trained to model student responses on the ARC benchmark experienced a notable decline. Specifically, the vicuna-7b model saw its accuracy decrease from 53.24\% to 40.61\% when trained on student dialogues. This significant drop in performance highlights a critical concern: training LLMs to replicate student behavior, including misconceptions, can severely impair their reasoning abilities.

However, our introduction of hallucination tokens into the training process presents a silver lining. Our approach led to a partial recovery in the ARC performance, with accuracy improving to 45.48\%. While this does not fully restore the model's baseline performance, it represents a significant step towards mitigating the regressive side effects of training LLMs on student data.

\textbf{Hallucination Evaluation (HaluEval) Dialogue Findings.}
The HaluEval Dial benchmark is designed to assess a model's ability to recognize and avoid hallucinations in generated responses, particularly in the context of knowledge grounded dialogue tasks. Hallucinations in this context refer to the model generating information that is not supported by the input data or general knowledge, a critical issue when models are used in educational settings where accuracy is paramount.
Our findings indicate that training models to model student responses led to a decrease in performance on the HaluEval Dial benchmark. Specifically, the vicuna-7b model saw its accuracy drop from 69.0\% to 65.39\%.
However, the introduction of hallucination tokens demonstrated a remarkable ability to counteract this effect, with the student-7b model's accuracy improving to 70.73\%. 

\textbf{Memorization Traps (MemoTrap) Findings.}
MemoTrap is a benchmark designed to test whether language models can avoid memorization traps by prompting them to complete well-known proverbs with endings that deviate from the commonly used ones. This benchmark is particularly relevant for evaluating a model's ability to generate creative and contextually appropriate responses rather than relying on rote memorization.

In our experiments, training models to model student responses resulted in a decrease in performance on the MemoTrap benchmark. The vicuna-7b model's accuracy decreased from 68.48\% to 65.28\%, indicating that training on student dialogues might encourage the model to rely more on memorization rather than understanding and applying knowledge flexibly.
The introduction of hallucination tokens led to a slight improvement, with accuracy increasing to 66.88\%.

\textbf{MMLU, HellaSwag, and Winogrande Findings.}
The performance of models on the MMLU, HellaSwag, and Winogrande benchmarks remained relatively stable, regardless of whether they were trained to model tutor or student responses.

The nuanced impact observed in other benchmarks underscores the importance of carefully considering the training data and methodologies used when developing LLMs for educational purposes. 
The introduction of hallucination tokens emerges as a promising strategy for mitigating some of the regressive side effects associated with training models to model student behavior, ensuring that they can still serve as effective tools for personalized learning without compromising on factual accuracy or reasoning capabilities.

\subsection{Control Models: Tutor Models}

To further understand the regressive side effects of training LLMs to model student behavior, we conducted a control experiment by training models to predict tutor responses. 
This experiment aimed to compare the performance of models trained to predict tutor responses versus those trained to predict student responses. 
The tutor models were trained using the same student-tutor dialogue dataset derived from the CLASS framework \cite{sonkar2023class}.
However, instead of training the models to model student responses, we trained them to predict the responses of the tutor.
Our findings, as shown in Table~\ref{tab:results}, revealed that training the LLMs on tutor responses did not lead to the same performance decline observed when modeling student responses.
This result underscores that the regressive side effects are a unique challenge specific to training LLMs to replicate student misconceptions.

%% file: tables/truthfulqa.tex
\begin{table*}[th!]
\centering
\caption{Performance of Vicuna models on TruthfulQA tasks. The table compares the performance of the original vicuna model, the control model trained to model tutor responses in biology (tutor) the model trained to model tutor responses in biology (tutor), and the model trained with hallucination tokens (student-\texttt{hal}). The results are presented for three different settings: MC1, MC2, and Generation. MC1 refers to a setting where there is only one correct answer to a question, while MC2 refers to a setting where there are multiple correct answers. For these settings, the performance is measured in terms of accuracy. The generation setting involves the model generating 1-2 sentence answers, with performance evaluated using BLEU and ROUGE scores. The results highlight the significant drop in performance when the model is trained to model student responses, demonstrating a regressive side effect in terms of truthfulness. However, the substantial recovery in performance with the introduction of hallucination tokens suggests a promising strategy to mitigate these regressive effects.}
\label{tab:truthful}
\resizebox{2\columnwidth}{!}{%
\begin{tabular}{lcccclc}
\multicolumn{1}{c}{\textbf{Dataset}} &
  \textbf{\begin{tabular}[c]{@{}c@{}}TQA MC1 \\ (Single-true)\end{tabular}} &
  \textbf{\begin{tabular}[c]{@{}c@{}}TQA MC2  \\ (Multi-true)\end{tabular}} &
  \multicolumn{4}{c}{\textbf{TruthfulQA (TQA) Generation}} \\ \hline
\multicolumn{1}{c}{\textbf{Metric}} &
  \textbf{Accuracy} &
  \textbf{Accuracy} &
  \textbf{BLEU} &
  \textbf{\begin{tabular}[c]{@{}c@{}}ROUGE \\ (unigram)\end{tabular}} &
  \textbf{\begin{tabular}[c]{@{}c@{}}ROUGE \\ (bigram)\end{tabular}} &
  \textbf{\begin{tabular}[c]{@{}c@{}}ROUGE \\ (LCS)\end{tabular}} \\ \hline
vicuna-7b-v1.5      & 32.93 & 50.37 & 49.69 & 51.41 & 45.90 & 50.55 \\
tutor-7b            & 34.64 & 52.43 & 42.72 & 47.12 & 37.94 & 45.29 \\ \hline
student-7b          & 23.75 & 36.14 & 24.60 & 29.74 & 14.32 & 28.89 \\
student-hal-7b  & 29.25 & 44.68 & 43.94 & 47.61 & 36.47 & 45.53 \\ \hline \hline
vicuna-13b-v1.5     & 35.01 & 50.87 & 47.12 & 50.18 & 44.92 & 49.08 \\
tutor-13b           & 34.76 & 52.20 & 42.84 & 48.71 & 38.80 & 46.76 \\ \hline
student-13b         & 22.15 & 33.93 & 15.18 & 18.12 & 6.12  & 17.75 \\
student-hal-13b & 27.91 & 41.46 & 39.29 & 42.35 & 33.66 & 42.96 \\ \hline \hline

\end{tabular}%
}
\end{table*}

%% file: tables/results_halu_mt.tex
\begin{table*}[t!]
\centering
\caption{Comparative performance of Large Language Models (LLMs) on various benchmarks before and after the introduction of hallucination tokens, with a control experiment involving tutor models. The table presents the performance of Vicuna 7B models across five key benchmarks: ARC Reasoning, Hallucination Evaluation Dialogue (HaluDial), Hallucination Memorization Trap (MemoTrap), TruthfulQA (TQA), HellaSwag (HSwag), MMLU, and Winogrande (WinoG). The numbers in parentheses (e.g., 25-S in ARC) represent the number of few-shot examples provided to the model during evaluation. The performance is measured in terms of accuracy percentage. The table compares the performance of the original vicuna models, tutor models, student models, and student models trained with hallucination tokens (student-\texttt{hal}). The results highlight the significant drop in performance when the model is trained to model student responses, demonstrating regressive side effects across multiple tasks. However, the introduction of hallucination tokens leads to a substantial recovery in performance across all benchmarks, underscoring their potential in mitigating these regressive effects.}
\resizebox{2\columnwidth}{!}{%
\begin{tabular}{|l|c|c|c|c|c|c|c|c|c}
\hline
\multicolumn{1}{|c|}{\textbf{Model}} 
& \textbf{Avg} & \textbf{\begin{tabular}[c]{@{}c@{}}ARC\\ (25-S)\end{tabular}} & \textbf{\begin{tabular}[c]{@{}c@{}}HaluDial\\ (0-S)\end{tabular}} & \textbf{\begin{tabular}[c]{@{}c@{}}MemoTrap\\ (0-S)\end{tabular}} & \textbf{\begin{tabular}[c]{@{}c@{}}TQA \\ (6-S)\end{tabular}} & \textbf{\begin{tabular}[c]{@{}c@{}c@{}}HSwag \\ (10-S)\end{tabular}} & \textbf{\begin{tabular}[c]{@{}c@{}}MMLU \\ (5-S)\end{tabular}} & \textbf{\begin{tabular}[c]{@{}c@{}}WinoG\\ (5-S)\end{tabular}} \\ \hline
vicuna-7b-v1.5 & 60.8 & 53.24 & 69.08 & 68.48 & 50.34 & 77.39 & 51.04 & 72.14 \\
tutor-7b & 61.0 & 52.13 & 68.81 & 69.23 & 52.3 & 78.07 & 51.32 & 71.19 \\
student-7b & 55.4 & 40.61 & 65.39 & 65.28 & 36.87 & 76.72 & 50.77 & 71.9 \\
student-\texttt{hal}-7b & 58.0 & 45.48 & 70.73 & 66.88 & 44.83 & 77.21 & 51.54 & 72.03 \\ \hline
vicuna-13b-v1.5 & 64.2 & 57.08 & 73.78 & 67.2 & 51.51 & 81.24 & 56.67 & 74.66 \\
tutor-13b & 64.7 & 57.34 & 73.92 & 66.13 & 52.99 & 81.51 & 57.02 & 74.35 \\
student-13b & 58.2 & 46.5 & 66.97 & 65.81 & 35.0 & 80.36 & 57.06 & 72.22 \\
student-\texttt{hal}-13b & 60.3 & 48.63 & 72.98 & 66.13 & 42.75 & 80.28 & 56.4 & 73.16 \\ \hline
\end{tabular}
}
\label{tab:results}
\end{table*}

%% file: sections/related_work.tex
The intersection of artificial intelligence and education has been an area of active research, with a focus on developing systems that can adapt to and support individual learners. Our work touches upon several research domains, including student modeling, the design of intelligent tutoring systems, and the deployment of Large Language Models (LLMs) in educational contexts.

\subsection{Student Modeling}
Student modeling has long been the cornerstone of personalized learning, with early attempts using rule-based and Bayesian systems to predict student knowledge and behaviors \cite{polson2013foundations}. 
Recent advancements have shifted towards utilizing machine learning to create more sophisticated models that can adapt to student learning patterns over time \cite{baker2009state,qdkt,liu2022okt,dupe}. 
Our work builds upon these foundations by exploring how LLMs can simulate not only the knowledge but also the typical errors and misconceptions students have during the learning process.

\subsection{Intelligent Tutoring Systems (ITS)}
Intelligent tutoring systems have been designed to provide immediate and personalized instruction or feedback to learners without human intervention \cite{woolf2010building}. The application of LLMs in ITS presents a novel opportunity to create systems that can engage in more natural and meaningful dialogues with students \cite{schmucker2023ruffle,code_class}. Our approach diverges from traditional ITS by focusing on the intentional generation of errors to mimic a student's learning trajectory, rather than solely providing expert-level instructions \cite{vanlehn2011relative}.

\subsection{Large Language Models in Education}
The use of LLMs like GPT \cite{gpt4sparkai} in education is a relatively new but rapidly growing field of study \cite{brown2020language,liu2024malalgoqa}. These models have been employed for various educational purposes, from generating educational content to serving as conversational agents \cite{heffernan2014assistments,sonkar2023class} to assessment \cite{marking,alag}. 
However, the challenge of ensuring the truthfulness and reliability of the information provided by LLMs is a recurring concern \cite{lin2021truthfulqa}. 
Our research contributes to this dialogue by investigating the impact of training LLMs to produce student-like errors and proposing a novel `hallucination token' to manage this trade-off.

\subsection{Truthfulness and Reliability in AI}
The TruthfulQA benchmark has been instrumental in highlighting the issues of truthfulness in AI-generated content \cite{clark2018think}. The ARC challenge further emphasizes the complexity of reasoning required from AI systems beyond simple fact retrieval \cite{etzioni2011open}. Our work is aligned with these challenges, as we seek to understand and improve the truthfulness and reasoning capacity of LLMs when they are trained to replicate student behaviors.

In conclusion, our study intersects with and contributes to the existing body of work in these areas by addressing the unique challenge of training LLMs to authentically mimic student learning processes, including the generation of errors. Our introduction of the ``hallucination token" represents a step forward in this domain, suggesting a new direction for future research and development.

%% file: sections/conclusion.tex
In this study, we have delved into the Student Data Paradox, a critical challenge that arises when training LLMs on student data for personalized education. Our findings reveal a complex trade-off: as LLMs become more adept at modeling student misconceptions, they tend to compromise their own factual integrity and reasoning abilities. We term this phenomenon the regressive side effects of the Student Data Paradox.
Our experiments demonstrated a notable decrease in the model's performance across various key benchmark datasets like ARC Reasoning Challenge and TruthfulQA. 
To mitigate these regressive side effects, we introduced a novel technique involving the use of hallucination tokens during the training process.
Our results indicate that the introduction of these tokens leads to a substantial improvement in the model's performance across all datasets.
However, it's important to note that despite the significant improvements achieved with the hallucination tokens, they do not fully restore the model's baseline performance. 
This outcome underscores the complexity of the problem and highlights the need for a more nuanced approach when training LLMs to mimic student behavior.
While we have made some strides in addressing the regressive side effects, our work is just the beginning.
We believe that our findings will pave the way for further research in this domain, ultimately contributing to the refinement of LLMs in personalized learning environments.

%% file: sections/limitations.tex
While our research provides valuable insights into the challenges of training LLMs on student data, there are some limitations to consider. Firstly, the impact of the Student Data Paradox on long-term learning outcomes remains an open question. Further longitudinal studies could shed light on how the trade-off between simulating student misconceptions and maintaining factual accuracy affects learners' progress over time.
Additionally, our study primarily focused on the technical aspects of LLM training and evaluation. Future research could delve into the pedagogical implications of using LLMs in personalized learning environments, exploring how educators can effectively integrate these models into their teaching practices.
Moreover, the hallucination token approach introduced in this paper, while promising, is just one potential solution to the Student Data Paradox. Continued research into alternative mitigation strategies could yield even more effective techniques for balancing the modeling of student behavior with the preservation of factual integrity.

%% file: sections/ethics.tex
Our research into the Student Data Paradox raises important ethical considerations for the development and deployment of LLMs in personalized education. As we have demonstrated, training LLMs on student data, while essential for creating adaptive learning systems, can lead to regressive side effects that compromise the models' factual accuracy and reasoning abilities. This poses a significant challenge for the responsible rollout of AI-driven educational products. However, our study also provides a path forward. By introducing hallucination tokens during the training process, we have shown that it is possible to mitigate these regressive effects substantially. This technique allows LLMs to differentiate between simulating student misconceptions and providing accurate information, a crucial step towards building trustworthy AI tutors. While our approach does not completely eliminate the paradox, it represents a significant advancement in the field. As such, our paper serves as a valuable resource to navigate the ethical complexities of developing personalized learning products. By building upon our findings and continuing to invest in research that addresses the paradox, one can responsibly harness the power of LLMs to revolutionize education. With right approach, we believe that AI-driven personalized learning can become a reality, providing students with adaptive support.

%% file: paper.bbl
\begin{thebibliography}{35}
\providecommand{\natexlab}[1]{#1}

\bibitem[{Baker et~al.(2009)Baker, Yacef et~al.}]{baker2009state}
Ryan~SJD Baker, Kalina Yacef, et~al. 2009.
\newblock The state of educational data mining in 2009: A review and future visions.
\newblock \emph{Journal of educational data mining}, 1(1):3--17.

\bibitem[{Brown et~al.(2020)Brown, Mann, Ryder, Subbiah, Kaplan, Dhariwal, Neelakantan, Shyam, Sastry, Askell et~al.}]{brown2020language}
Tom Brown, Benjamin Mann, Nick Ryder, Melanie Subbiah, Jared~D Kaplan, Prafulla Dhariwal, Arvind Neelakantan, Pranav Shyam, Girish Sastry, Amanda Askell, et~al. 2020.
\newblock Language models are few-shot learners.
\newblock \emph{Advances in neural information processing systems}, 33:1877--1901.

\bibitem[{Bubeck et~al.(2023)Bubeck, Chandrasekaran, Eldan, Gehrke, Horvitz, Kamar, Lee, Lee, Li, Lundberg et~al.}]{gpt4sparkai}
S{\'e}bastien Bubeck, Varun Chandrasekaran, Ronen Eldan, Johannes Gehrke, Eric Horvitz, Ece Kamar, Peter Lee, Yin~Tat Lee, Yuanzhi Li, Scott Lundberg, et~al. 2023.
\newblock {Sparks of Artificial General Intelligence: Early experiments with GPT-4}.
\newblock \emph{arXiv preprint arXiv:2303.12712}.

\bibitem[{Chiang et~al.(2023)Chiang, Li, Lin, Sheng, Wu, Zhang, Zheng, Zhuang, Zhuang, Gonzalez, Stoica, and Xing}]{vicuna2023}
Wei-Lin Chiang, Zhuohan Li, Zi~Lin, Ying Sheng, Zhanghao Wu, Hao Zhang, Lianmin Zheng, Siyuan Zhuang, Yonghao Zhuang, Joseph~E. Gonzalez, Ion Stoica, and Eric~P. Xing. 2023.
\newblock \href {https://lmsys.org/blog/2023-03-30-vicuna/} {Vicuna: An open-source chatbot impressing gpt-4 with 90\%* chatgpt quality}.

\bibitem[{Clark et~al.(2018{\natexlab{a}})Clark, Cowhey, Etzioni, Khot, Sabharwal, Schoenick, and Tafjord}]{clark2018thinkarc}
Peter Clark, Isaac Cowhey, Oren Etzioni, Tushar Khot, Ashish Sabharwal, Carissa Schoenick, and Oyvind Tafjord. 2018{\natexlab{a}}.
\newblock \href {https://arxiv.org/abs/1803.05457} {Think you have solved question answering? try arc, the ai2 reasoning challenge}.
\newblock \emph{Preprint}, arXiv:1803.05457.

\bibitem[{Clark et~al.(2018{\natexlab{b}})Clark, Cowhey, Etzioni, Khot, Sabharwal, Schoenick, and Tafjord}]{clark2018think}
Peter Clark, Isaac Cowhey, Oren Etzioni, Tushar Khot, Ashish Sabharwal, Carissa Schoenick, and Oyvind Tafjord. 2018{\natexlab{b}}.
\newblock Think you have solved question answering? try arc, the ai2 reasoning challenge.
\newblock \emph{arXiv preprint arXiv:1803.05457}.

\bibitem[{Etzioni et~al.(2011)Etzioni, Fader, Christensen, Soderland, and Mausam}]{etzioni2011open}
Oren Etzioni, Anthony Fader, Janara Christensen, Stephen Soderland, and Mausam Mausam. 2011.
\newblock Open information extraction: The second generation.
\newblock In \emph{IJCAI}, volume~11, pages 3--10.

\bibitem[{Gao et~al.(2023)Gao, Tow, Abbasi, Biderman, Black, DiPofi, Foster, Golding, Hsu, Le~Noac'h, Li, McDonell, Muennighoff, Ociepa, Phang, Reynolds, Schoelkopf, Skowron, Sutawika, Tang, Thite, Wang, Wang, and Zou}]{eval-harness}
Leo Gao, Jonathan Tow, Baber Abbasi, Stella Biderman, Sid Black, Anthony DiPofi, Charles Foster, Laurence Golding, Jeffrey Hsu, Alain Le~Noac'h, Haonan Li, Kyle McDonell, Niklas Muennighoff, Chris Ociepa, Jason Phang, Laria Reynolds, Hailey Schoelkopf, Aviya Skowron, Lintang Sutawika, Eric Tang, Anish Thite, Ben Wang, Kevin Wang, and Andy Zou. 2023.
\newblock \href {https://doi.org/10.5281/zenodo.10256836} {A framework for few-shot language model evaluation}.

\bibitem[{Goldberg(2022)}]{goldbergAssessingClaims2022}
Yoav Goldberg. 2022.
\newblock Assessing claims about large language models.
\newblock \emph{arXiv preprint arXiv:2212.09273}.

\bibitem[{Heffernan and Heffernan(2014)}]{heffernan2014assistments}
Neil~T Heffernan and Cristina~Lindquist Heffernan. 2014.
\newblock The assistments ecosystem: Building a platform that brings scientists and teachers together for minimally invasive research on human learning and teaching.
\newblock \emph{International Journal of Artificial Intelligence in Education}, 24:470--497.

\bibitem[{Hendrycks et~al.(2020)Hendrycks, Burns, Basart, Zou, Mazeika, Song, and Steinhardt}]{mmlu}
Dan Hendrycks, Collin Burns, Steven Basart, Andy Zou, Mantas Mazeika, Dawn Song, and Jacob Steinhardt. 2020.
\newblock \href {https://arxiv.org/abs/2009.03300} {Measuring massive multitask language understanding}.
\newblock \emph{CoRR}, abs/2009.03300.

\bibitem[{Jurenka et~al.(2024)}]{learnlm}
Irina Jurenka et~al. 2024.
\newblock Towards responsible development of generative ai for education: An evaluation-driven approach.
\newblock Technical report, Google.

\bibitem[{Li et~al.(2023)Li, Cheng, Zhao, Nie, and Wen}]{li-etal-2023-halueval}
Junyi Li, Xiaoxue Cheng, Xin Zhao, Jian-Yun Nie, and Ji-Rong Wen. 2023.
\newblock \href {https://doi.org/10.18653/v1/2023.emnlp-main.397} {{H}alu{E}val: A large-scale hallucination evaluation benchmark for large language models}.
\newblock In \emph{Proceedings of the 2023 Conference on Empirical Methods in Natural Language Processing}, pages 6449--6464, Singapore. Association for Computational Linguistics.

\bibitem[{Lin et~al.(2021)Lin, Hilton, and Evans}]{lin2021truthfulqa}
Stephanie Lin, Jacob Hilton, and Owain Evans. 2021.
\newblock Truthfulqa: Measuring how models mimic human falsehoods.
\newblock \emph{arXiv preprint arXiv:2109.07958}.

\bibitem[{Lin et~al.(2022)Lin, Hilton, and Evans}]{lin2022truthfulqa}
Stephanie Lin, Jacob Hilton, and Owain Evans. 2022.
\newblock \href {https://arxiv.org/abs/2109.07958} {Truthfulqa: Measuring how models mimic human falsehoods}.
\newblock \emph{Preprint}, arXiv:2109.07958.

\bibitem[{Liu et~al.(2024)Liu, Sonkar, Le, and Baraniuk}]{liu2024malalgoqa}
Naiming Liu, Shashank Sonkar, Myco Le, and Richard Baraniuk. 2024.
\newblock {MalAlgoQA: A Pedagogical Approach for Evaluating Counterfactual Reasoning Abilities}.
\newblock \emph{arXiv preprint arXiv:2407.00938}.

\bibitem[{Liu et~al.(2023)Liu, Sonkar, Wang, Woodhead, and Baraniuk}]{nlet}
Naiming Liu, Shashank Sonkar, Zichao Wang, Simon Woodhead, and Richard~G Baraniuk. 2023.
\newblock Novice learner and expert tutor: Evaluating math reasoning abilities of large language models with misconceptions.
\newblock \emph{arXiv preprint arXiv:2310.02439}.

\bibitem[{Liu et~al.(2022)Liu, Wang, Baraniuk, and Lan}]{liu2022okt}
Naiming Liu, Zichao Wang, Richard Baraniuk, and Andrew Lan. 2022.
\newblock Open-ended knowledge tracing for computer science education.
\newblock In \emph{Proceedings of the 2022 Conference on Empirical Methods in Natural Language Processing}, pages 3849--3862.

\bibitem[{McKenzie et~al.(2022)McKenzie, Lyzhov, Parrish, Prabhu, Mueller, Kim, Bowman, and Perez}]{mckenzie2022inverse}
Ian McKenzie, Alexander Lyzhov, Alicia Parrish, Ameya Prabhu, Aaron Mueller, Najoung Kim, Sam Bowman, and Ethan Perez. 2022.
\newblock \href {https://github.com/inverse-scaling/prize} {The inverse scaling prize}.

\bibitem[{OpenAI(2024)}]{openai2024chatgptedu}
OpenAI. 2024.
\newblock Introducing chatgpt edu.
\newblock \url{https://openai.com/index/introducing-chatgpt-edu/}.
\newblock Accessed: 2022-04-28.

\bibitem[{Polson and Richardson(2013)}]{polson2013foundations}
Martha~C Polson and J~Jeffrey Richardson. 2013.
\newblock \emph{Foundations of intelligent tutoring systems}.
\newblock Psychology Press.

\bibitem[{Sakaguchi et~al.(2019)Sakaguchi, Bras, Bhagavatula, and Choi}]{WINOGRANDE}
Keisuke Sakaguchi, Ronan~Le Bras, Chandra Bhagavatula, and Yejin Choi. 2019.
\newblock \href {https://arxiv.org/abs/1907.10641} {{WINOGRANDE:} an adversarial winograd schema challenge at scale}.
\newblock \emph{Preprint}, arXiv:1907.10641.

\bibitem[{Samaniego et~al.(2022)Samaniego, Sap, Ammanabrolu, Reif, and Choi}]{samaniegoSanityCheckingLarge2022}
Marta Samaniego, Maarten Sap, Prithviraj Ammanabrolu, Emily Reif, and Yejin Choi. 2022.
\newblock Sanity checking large language models.
\newblock \emph{arXiv preprint arXiv:2212.10465}.

\bibitem[{Schmucker et~al.(2023)Schmucker, Xia, Azaria, and Mitchell}]{schmucker2023ruffle}
Robin Schmucker, Meng Xia, Amos Azaria, and Tom Mitchell. 2023.
\newblock Ruffle\&riley: Towards the automated induction of conversational tutoring systems.
\newblock \emph{arXiv preprint arXiv:2310.01420}.

\bibitem[{Sonkar and Baraniuk(2023)}]{dupe}
Shashank Sonkar and Richard~G Baraniuk. 2023.
\newblock {Deduction under Perturbed Evidence: Probing Student Simulation (Knowledge Tracing) Capabilities of Large Language Models.}
\newblock In \emph{LLM@ AIED}, pages 26--33.

\bibitem[{Sonkar et~al.(2024{\natexlab{a}})Sonkar, Chen, Le, Liu, Basu~Mallick, and Baraniuk}]{code_class}
Shashank Sonkar, Xinghe Chen, Myco Le, Naiming Liu, Debshila Basu~Mallick, and Richard Baraniuk. 2024{\natexlab{a}}.
\newblock \href {https://doi.org/10.1145/3636555.3636889} {{Code Soliloquies for Accurate Calculations in Large Language Models}}.
\newblock In \emph{Proceedings of the 14th Learning Analytics and Knowledge Conference}, LAK '24, page 828–835, New York, NY, USA. Association for Computing Machinery.

\bibitem[{Sonkar et~al.(2020)Sonkar, Lan, Waters, Grimaldi, and Baraniuk}]{qdkt}
Shashank Sonkar, Andrew~S. Lan, Andrew~E. Waters, Phillip Grimaldi, and Richard~G. Baraniuk. 2020.
\newblock \href {https://educationaldatamining.org/files/conferences/EDM2020/papers/paper_35.pdf} {{qDKT: Question-centric Deep Knowledge Tracing}}.
\newblock In \emph{Proceedings of the 13th International Conference on Educational Data Mining, EDM 2020, Fully virtual conference, July 10-13, 2020}. International Educational Data Mining Society.

\bibitem[{Sonkar et~al.(2023)Sonkar, Liu, Mallick, and Baraniuk}]{sonkar2023class}
Shashank Sonkar, Naiming Liu, Debshila Mallick, and Richard Baraniuk. 2023.
\newblock \href {https://aclanthology.org/2023.findings-emnlp.130} {{{CLASS}: A Design Framework for Building Intelligent Tutoring Systems Based on Learning Science principles}}.
\newblock In \emph{Findings of the Association for Computational Linguistics: EMNLP 2023}, pages 1941--1961, Singapore.

\bibitem[{Sonkar et~al.(2024{\natexlab{b}})Sonkar, Liu, Mallick, and Baraniuk}]{marking}
Shashank Sonkar, Naiming Liu, Debshila~B. Mallick, and Richard~G. Baraniuk. 2024{\natexlab{b}}.
\newblock {{Marking: Visual Grading with Highlighting Errors and Annotating Missing Bits}}.
\newblock In \emph{{{Artificial Intelligence in Education}}}, pages 309--323, Cham. Springer Nature Switzerland.

\bibitem[{Sonkar et~al.(2024{\natexlab{c}})Sonkar, Ni, Chaudhary, and Baraniuk}]{sonkar2024pedagogical}
Shashank Sonkar, Kangqi Ni, Sapana Chaudhary, and Richard~G Baraniuk. 2024{\natexlab{c}}.
\newblock Pedagogical alignment of large language models.
\newblock \emph{arXiv preprint arXiv:2402.05000}.

\bibitem[{Sonkar et~al.(2024{\natexlab{d}})Sonkar, Ni, Tran~Lu, Kincaid, Hutchinson, and Baraniuk}]{alag}
Shashank Sonkar, Kangqi Ni, Lesa Tran~Lu, Kristi Kincaid, John~S. Hutchinson, and Richard~G. Baraniuk. 2024{\natexlab{d}}.
\newblock {{Automated Long Answer Grading with RiceChem Dataset}}.
\newblock In \emph{Artificial Intelligence in Education}, pages 163--176, Cham. Springer Nature Switzerland.

\bibitem[{Touvron et~al.(2023)Touvron, Martin, Stone, Albert, Almahairi, Babaei, Bashlykov, Batra, Bhargava, Bhosale et~al.}]{llama2}
Hugo Touvron, Louis Martin, Kevin Stone, Peter Albert, Amjad Almahairi, Yasmine Babaei, Nikolay Bashlykov, Soumya Batra, Prajjwal Bhargava, Shruti Bhosale, et~al. 2023.
\newblock Llama 2: Open foundation and fine-tuned chat models.
\newblock \emph{arXiv preprint arXiv:2307.09288}.

\bibitem[{VanLehn(2011)}]{vanlehn2011relative}
Kurt VanLehn. 2011.
\newblock The relative effectiveness of human tutoring, intelligent tutoring systems, and other tutoring systems.
\newblock \emph{Educational psychologist}, 46(4):197--221.

\bibitem[{Woolf(2010)}]{woolf2010building}
Beverly~Park Woolf. 2010.
\newblock \emph{Building intelligent interactive tutors: Student-centered strategies for revolutionizing e-learning}.
\newblock Morgan Kaufmann.

\bibitem[{Zellers et~al.(2019)Zellers, Holtzman, Bisk, Farhadi, and Choi}]{zellers2019hellaswag}
Rowan Zellers, Ari Holtzman, Yonatan Bisk, Ali Farhadi, and Yejin Choi. 2019.
\newblock \href {https://arxiv.org/abs/1905.07830} {Hellaswag: Can a machine really finish your sentence?}
\newblock \emph{Preprint}, arXiv:1905.07830.

\end{thebibliography}
